\def\year{2021}\relax
\documentclass[letterpaper]{article} 
\usepackage{aaai21}  
\usepackage{times}  
\usepackage{helvet} 
\usepackage{courier}  
\usepackage[hyphens]{url}  
\usepackage{graphicx} 
\usepackage{natbib}  
\usepackage{caption} 
\usepackage[utf8]{inputenc} 
\usepackage[T1]{fontenc}    
\usepackage{url}            
\usepackage{booktabs}       
\usepackage{amsfonts}       
\usepackage{nicefrac}       
\usepackage{microtype}      

\usepackage{tabularx}
\usepackage{float}
\usepackage[caption=false]{subfig}
\usepackage{epsfig}
\usepackage{epstopdf}
\usepackage{bbm}
\usepackage{mathtools,xparse}

\usepackage{epsfig} 
\usepackage{amsmath} 
\usepackage{amssymb}  
\usepackage{bm}
\usepackage{color}
\usepackage{etoolbox}
\DeclareMathOperator*{\argmax}{arg\,max}

\usepackage{comment}
\usepackage{diagbox}
\usepackage[ruled, lined, longend]{algorithm2e}
\usepackage{booktabs,multirow}
\usepackage{bigstrut}

\usepackage{lipsum}
\usepackage{color}
\usepackage[colorinlistoftodos]{todonotes}

\usepackage{multicol}
\usepackage{hyperref}
\usepackage{amsthm}

\usepackage{amsfonts}
\usepackage{mathtools}

\usepackage{graphicx}
\usepackage[capitalise]{cleveref}
\usepackage{hyperref}

\newtheorem{theorem}{Theorem}

\newtheorem{definition}{Definition}

\title{Bounded Risk-Sensitive Markov Games: Forward Policy Design and Inverse Reward Learning with Iterative Reasoning and Cumulative Prospect Theory}
\author{
 Ran Tian \textsuperscript{ \footnotesize \textasteriskcentered \normalsize},
 Liting Sun \footnote{First two authors contributed equally to this work.},
 Masayoshi Tomizuka\\
}

\affiliations{
University of California, Berkeley\\
\{rantian, litingsun, tomizuka\}@berkeley.edu 
}

\begin{document}

\maketitle

\begin{abstract}
Classical game-theoretic approaches for multi-agent systems in both the forward policy design problem and the inverse reward learning problem often make strong rationality assumptions: agents perfectly maximize expected utilities under uncertainties. Such assumptions, however, substantially mismatch with observed human behaviors such as satisficing with sub-optimal, risk-seeking, and loss-aversion decisions. Drawing on iterative reasoning models and cumulative prospect theory, we propose a new game-theoretic framework, bounded risk-sensitive Markov Game (BRSMG), that captures two aspects of realistic human behaviors: bounded intelligence and risk-sensitivity. General solutions to both the forward policy design problem and the inverse reward learning problem are provided with theoretical analysis and simulation verification. We validate the proposed forward policy design algorithm and the inverse reward learning algorithm in a two-player navigation scenario. The results show that agents demonstrate bounded-intelligence, risk-averse and risk-seeking behaviors in our framework. Moreover, in the inverse reward learning task, the proposed bounded risk-sensitive inverse learning algorithm outperforms a baseline risk-neutral inverse learning algorithm by effectively learning not only more accurate reward values but also the intelligence levels and the risk-measure parameters of agents from demonstrations.
\end{abstract}

\section{Introduction} \label{sec: intro}

Markov Game (MG), as an approach to modeling interactions and decision-making processes in multi-agent systems, has been employed in many domains such as economics \cite{amir2003stochastic}, games \cite{silver2017mastering}, and human-robot/machine interaction \cite{bu2008comprehensive}. In classical MGs, agents are commonly assumed to be perfectly rational when computing their policies. For instance, in a two-player MG, agent $1$ is assumed to make decisions based on his/her belief in agent $2$'s behavioral model in which agent $2$ is also assumed to behave according to his/her belief in agent $1$'s model $\dots$ and both agents are maximizing their expected rewards based on such infinite levels of mutual beliefs. If the beliefs match the actual models, perfect Markov strategies of all agents may be found by solving the Markov-perfect equilibrium of the game where a Nash equilibrium is reached. Under such assumptions, we can either solve for humans' optimal policies with handcrafted rewards (forward policy design) or learn humans' rewards from demonstrations (inverse reward learning).

However, in real life, humans often significantly deviate from such ``perfectly rational'' assumptions from two major aspects \cite{goeree2001ten}. First, mounting evidence has shown that rather than spending a great amount of effort hunting for the best response, humans often choose actions that are satisfying (i.e., actions that are above their pre-defined thresholds according to certain criteria) and relatively quick and easy to find. Simon \cite{simon1976substantive} formulated such a cognitive characteristic as bounded rationality. Among the many developed behavioral models that capture bounded rationality, iterative reasoning models from behavioral game theory \cite{camerer2011behavioral} are some of the most prominent paradigms. These models do not assume humans perform infinite layers of strategic thinking during interactions but model humans as agents with finite levels of intelligence (bounded rationality). Second, instead of optimizing risk-neutral rewards, humans demonstrate a strong tendency towards risk-sensitive measures when evaluating the outcomes of their actions. They are risk-seeking in terms of gains and risk-averse for losses. Such deviations make it difficult to model realistic human behaviors using classical MGs.

In this work, we aim to establish a new game-theoretic framework (BRSMG) that captures the two aspects of realistic human behaviors discussed above. The incorporation of bounded rationality and risk-sensitivity in classical MGs requires revisiting fundamental concepts in both the forward policy design and the inverse reward learning problem. Standard value iteration and inverse learning algorithms for traditional MGs do not hold any more, and new algorithms should be established to reflect the impact of bounded intelligence and risk sensitivity. 

More specifically, in the forward policy design problem, we model humans' bounded intelligence via an instantiation of iterative reasoning models and model the influence of humans' risk sensitivity via cumulative prospect theory (CPT) \cite{tversky1992advances}. In the inverse reward learning problem, we develop a bounded risk-sensitive inverse learning algorithm that can recover not only the nominal rewards of agents but also their intelligence levels and risk-measure parameters from demonstrations. \textit{To our best knowledge, our work is the first to incorporate both bounded rationality and risk-sensitivity in both the forward problem and the inverse problem of general-sum MGs.}

\noindent
\textbf{Contributions.} In summary, our contributions are threefold:
\begin{enumerate}

\item We propose a novel game-theoretic framework (BRSMG) that captures bounded rationality and risk-sensitivity in humans' reasoning processes.

\item The proposed framework makes the first attempt to establish a bridge between inverse reward learning and risk-sensitive iterative reasoning models.

\item In contrast to previous risk-neutral reward learning algorithms that learn humans' rewards under equilibrium solution concepts, we exploit an alternative paradigm based on non-equilibrium solution concepts and offer a solution that simultaneously learns humans' rewards, intelligence levels, and risk-sensitive measure parameters. Therefore, our solution provides an interpretable and heterogeneous human behavioral model, which is of critical importance for the development of human-centered robots such as autonomous vehicles.
\end{enumerate}

\section{Related Work}\label{sec: related work}
\noindent
\textbf{Bounded rationality.} The influence of bounded rationality in forward policy design problems has been studied in both single-agent and multi-agent settings. One group of studies formulates such influence by introducing additional computational costs to agents' actions \cite{ben2007approach, halpern2008beyond, halpern2015algorithmic}. Another group focuses on models that can explicitly capture the bounded reasoning processes of humans. Examples include the Boltzmann rationality model \cite{von2007theory}, the quantal response equilibrium solution (QRE) \cite{mckelvey1995quantal}, and various iterative reasoning models \cite{costa2001cognition, camerer2004cognitive, stahl1994experimental}. The Boltzmann model and the QRE model formulate irrational behaviors of humans via sub-optimality, while iterative reasoning models emphasize more on the bounded reasoning depth. Instead of assuming humans perform infinite levels of strategic reasoning, iterative reasoning models only allow for a finite number of strategic reasoning. Iterative reasoning models have been exploited for modeling human behaviors in many application domains, including normal-form zero-sum games \cite{2020Beating}, aerospace \cite{yildiz2014predicting, 9147737}, cyber-physical security \cite{kanellopoulos2019non}, and human-robot interaction \cite{Li2018highway, tian2020game}. It is shown in \cite{wright2014level} that compared to QRE, iterative reasoning models can achieve better performance in predicting human behaviors in simultaneous move games. More specifically, \cite{WRIGHT201716} suggests that the quantal level-$k$ model is the state-of-the-art among various iterative reasoning models.

\noindent
\textbf{Risk measure.} Many risk measures have been proposed to evaluate humans' decisions. Beyond expectation, value-at-risk (VaR) and conditional value-at-Risk (CVaR) \cite{pflug2000some} are two well-adopted risk measures. In addition, the cumulative prospect theory (CPT) \cite{tversky1992advances} formulates a model that can well explain a substantial amount of human risk-sensitive behaviors. In the light of those risk measures, many risk-aware decision-making and reward learning algorithms have been proposed in both single-agent setting \cite{lin2013dynamic, chow2015risk, Mazumdarrisk, jie2018stochastic, ratliff2019inverse, kwon2020humans} and multi-agents cases \cite{sun2019interpretable} with a Stackelberg Game assumption.

\noindent
\textbf{Inverse reward learning in games.} The inverse reward learning problem in games has attracted great attention from researchers, starting from simplified open-loop game formulations \cite{sadigh2016planning, sun2018courteous} to closed-loop games \cite{yu2019multi, gruver2020multi}. The concept of QRE was first adopted by \cite{yu2019multi} to extend the maximum-entropy inverse reinforcement learning algorithm \cite{ziebart2008maximum} in multi-agent settings. \cite{gruver2020multi} further extended the idea for better efficiency and scalability by introducing a latent space in the reward network. Though \cite{wright2014level} suggested that iterative reasoning models can predict human behaviors more accurately in simultaneous move games compared with QRE, the multi-agent inverse reward learning problem with iterative reasoning models and risk sensitive measure has never been addressed. In this work, we propose the BRSMG framework to fill the gap.
\section{Preliminaries}
\subsection{Classical Markov Game} \label{sec: Markov Game}

In this work, we consider two-player Markov Games. We denote a two-player MG as $\mathcal{G} \triangleq \langle\mathcal{P}, \mathcal{S}, \mathcal{A}, \mathcal{R}, \mathcal{T}, \Tilde{\gamma}\rangle$, where $\mathcal{P} {=} \{1, 2\}$ is the set of agents in the game; $\mathcal{S} {=} S^1\times S^2$ and $\mathcal{A} {=} A^1 \times A^2$ are, respectively, the joint state and action spaces of the two agents; $\mathcal{R} {=} (R^1, R^2)$ is the set of agents' one-step reward functions with $R^{i}: \mathcal{S} \times A^i \times A^{-i} \to \mathbb{R}$ ($-i = \mathcal{P} \setminus \{i\}$ represents the opponent of agent $i$); $\mathcal{T}: \mathcal{S} \times \mathcal{A} \to \mathcal{S}$ represents the state transition of the game (we consider deterministic state transitions in this paper); and $\tilde{\gamma}$ is the reward discount factor.

We let $\pi^i: \mathcal{S}  \to A^i$ denote a deterministic policy of agent $i$. At step $t$, given the current state $s_t$, each agent tries to maximize its expected total discounted reward under uncertainties in its opponent's responses. Namely, the optimal policy $\pi^{*,i}$ is given by $\pi^{*,i}{=}\arg\max_{\pi^i} V^{i,\pi^i}(s_t)$, where \small$ V^{i,\pi^i}(s_t)=\mathbb{E}_{\pi^{-i}} \Big[ \sum_{\tau = 0} ^ {\infty} \tilde{\gamma}^{\tau} R^{i} (s_{t+\tau},a^{i}_{t+\tau},a^{-i}_{t+\tau}) \Big]$ \normalsize represents the {value of} $s_t$, i.e., the expected total reward collected by $i$ starting from $s_t$ under policy $\pi^i$. 
The notations $a^{-i}_{t+\tau}$ and $s_{t+\tau}$, respectively, represent the predicted future action of $-i$ and state of the game at step $t+\tau$. In the MPE, both agents achieve their optimal policies. Due to the mutual influence between the value functions of both agents, finding the MPE is normally NP-hard.

\subsection{Quantal Level-k Model}
\label{sec: CHT}

The quantal level-$k$ model is one of the most effective iterative reasoning models in predicting human behaviors in simultaneous move games \cite{WRIGHT201716}. It assumes that each human agent has an \textit{intelligence level} that defines his/her reasoning capability. More specifically, the level-$0$ agents do not perform any strategic reasoning, while quantal level-$k$ ($k\ge1$) agents make strategic decisions by treating other agents as quantal level-$(k{-}1)$ agents. As shown in \cref{fig: intro_figure}, the orange agent is a level-$1$ agent who believes that the blue agent is a level-$0$ agent. Meanwhile, the blue agent is actually a level-$2$ agent who treats the orange agent as a level-$1$ agent when making decisions. The quantal level-$k$ model has therefore reduced the complex circular strategic thinking in classical MGs to finite levels of iterative optimizations. On the basis of an anchoring level-$0$ policy, the quantal level-$k$ policies of all agents can be defined for all $k = 1,\dots,k_{\text{max}}$ through a sequential and iterative process.

\subsection{Cumulative Prospect Theory}\label{sec: CPT}
The cumulative prospect theory (CPT) is a non-expected utility measure that describes the risk-sensitivity of humans' decision-making processes \cite{kahneman2013prospect}. It can explain many systematic biases of human behaviors that deviate from risk-neutral decisions, such as risk-avoiding/seeking and framing effects.

\begin{definition}[CPT value]
	For a discrete random variable $X$ satisfying $\sum_{i{=}{-}m}^n\mathbb{P}(X{=}x_i){=}1$, $x_i{\ge} x^0$ for $i{=}0,{\cdots}, n$, and $x_i{<}x^0$ for $i{=}{-}m,{\cdots}, {-}1$, then the CPT value of $X$ is defined as
\small
\begin{subequations}
\label{equ:CPT_measure_discrete}
\begin{align}
& \mathbb{CPT}(X) {=}{\sum}_{i{=}0}^{n}\tilde{\rho}^+\left(\mathbb{P}(X{=}x_i)\right)u^+(X-x^0)\nonumber\\
&\hspace{+1cm} -{\sum}_{i{=}{-}m}^{{-}1}\tilde{\rho}^-\left(\mathbb{P}(X{=}x_i)\right)u^-(X-x^0),\\
& \tilde{\rho}^+\left(\mathbb{P}(X{=}x_i)\right) =\big[w^+\left({\sum}_{j{=}i}^{n}\mathbb{P}(X{=}x_j)\right)\nonumber\\
& \hspace{+2.5cm}{-}w^+\left({\sum}_{j{=}i{+}1}^{n}\mathbb{P}(X{=}x_j)\right)\big],\\
& \tilde{\rho}^-\left(\mathbb{P}(X{=}x_i)\right) =\big[w^-\left({\sum}_{j{=}{-}m}^{i}\mathbb{P}(X{=}x_j)\right)\nonumber\\
&\hspace{+2.5cm}{-}w^-\left({\sum}_{j{=}{-}m}^{i{-}1}\mathbb{P}(X{=}x_j)\right)\big].
\end{align}
\end{subequations}
\end{definition}
\normalsize
The functions $w^+: [0,1]{\to} [0,1]$ and $w^-: [0,1]{\to} [0,1]$ are two continuous non-decreasing functions and are referred as the probability weighting functions. They describe humans' desire to deflate high probabilities and inflate low probabilities. The two functions $u^+: \mathbb{R}^+{\to}\mathbb{R}^+$ and $u^-: \mathbb{R}^- {\to} \mathbb{R}^+$ are concave utility functions which are, respectively, monotonically non-decreasing and non-increasing. The notation $x^{0}$ denotes a reference point that separates the value $X$ into gains ($X{\geq} x^0$) and losses ($X{<}x^0$). 
Without loss of generality, we set $x^{0} = 0$ and omit $x^{0}$ in the rest of this paper. 

Many experimental studies have shown that representative functional forms for $u$ and $w$ are: $u^+(x) {=} (x)^\alpha$ if $x {\geq} 0$, and $u^+(x) {=} 0$ otherwise; $u^-(x)  {=}\lambda(-x)^\beta$ if $x {<} 0$, and $u^-(x)  {=} 0$ otherwise; $w^+(p) {=} \frac{p^\gamma}{\left(p^\gamma+(1-p)^\gamma\right)^{1/\gamma}}$ and $w^-(p) = \frac{p^\delta}{\left(p^\delta+(1-p)^\delta\right)^{1/\delta}}$. The parameters $\alpha, \beta, \gamma, \delta{\in}(0,1]$ are model parameters. We adopt these two representative functions in this paper. 
Section A of the supplementary material illustrates the probability weighting functions and the utility functions. 

\begin{figure}
\centering
\includegraphics[trim=0.2cm 0.8cm 0.2cm 0cm,clip,width=0.98\linewidth]{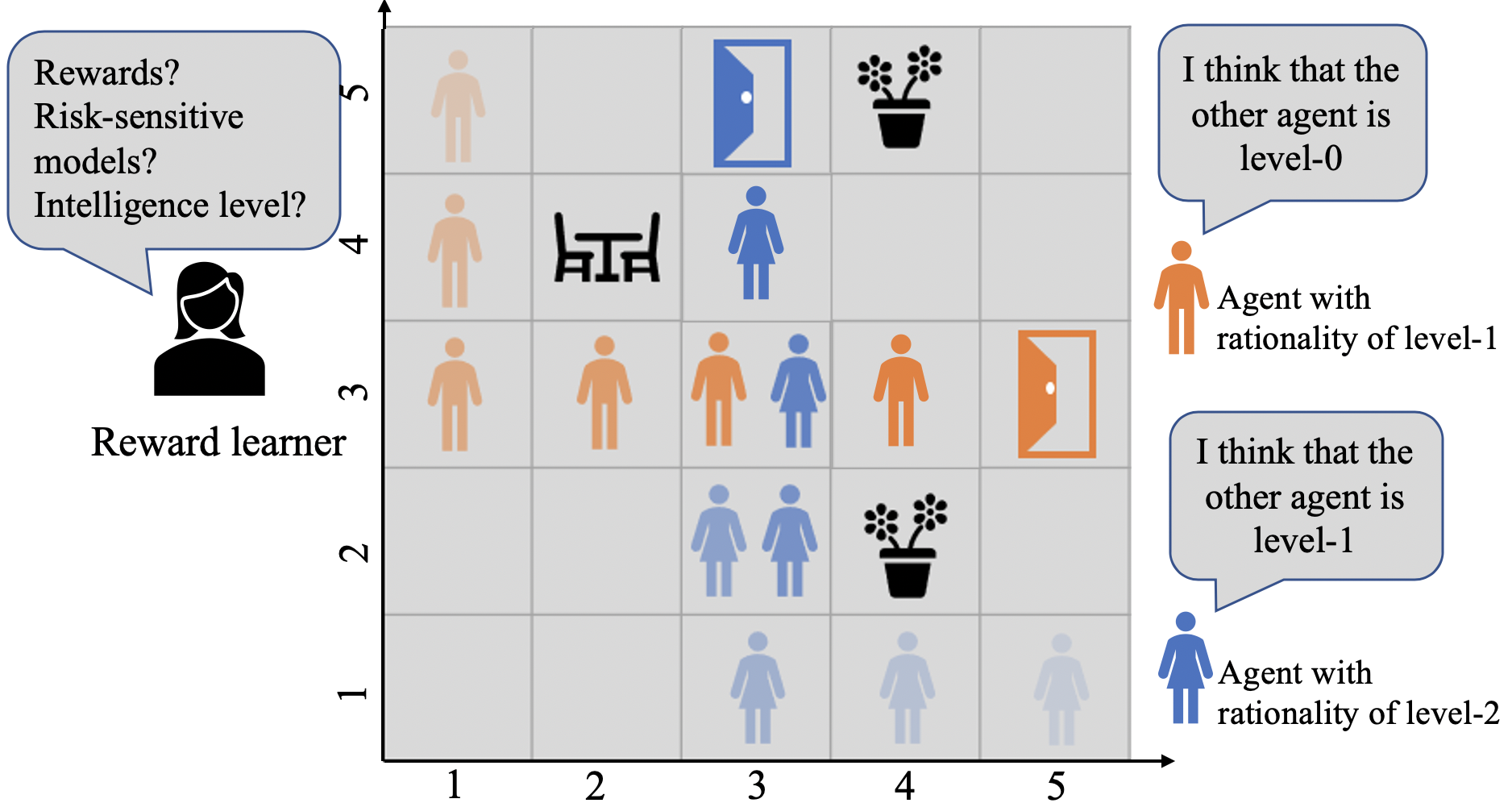}
\caption{ Modeling interactions between humans as a bounded risk-sensitive Markov Game: two human agents aim to exit the room through specified doors without collisions with obstacles and each other. We aim to answer two questions: 1) assuming both humans have bounded intelligence levels and risk-sensitive performance measures, how will their optimal policies differ from those in classical MGs? and 2) how to recover the rewards, intelligence levels, and risk-sensitivity parameters from their demonstrations?
\normalsize}
\label{fig: intro_figure}
\end{figure}
\section{Bounded Risk-Sensitive Markov Game}
\label{sec: BRSMG}

In this section, we investigate agents' policies in a new general-sum two-player MG, i.e., the bounded risk-sensitive MG (BRSMG). In particular, agents in BRSMG are bounded-rational with risk-sensitive performance measures. 

\subsection{Bounded Risk-Sensitive Policies}\label{sec: BRSMG policy}

According to the quantal level-$k$ model described in \cref{sec: CHT}, a quantal level-$k$ agent ($k{\in}\mathbb{N}^+$) assumes its opponent agent is quantal level-$(k-1)$ agent, predicts its quantal level-$(k-1)$ policy, and quantally best responds to the quantal level-$(k-1)$ policy. Such an iterative reasoning process traces back to the quantal level-$0$ policy, which is normally a pure responsive policy. Therefore, on the basis of a selected quantal level-$0$ policy\footnote{Note that the selection of quantal level-$0$ policy can be different according to applications. We use the notation $\pi^{0}$ to represent a generic quantal level-$0$ policy and describe the exemplary quantal level-$0$ policy in \cref{sec: simulation}.}, we can sequentially and iteratively solve for the closed-loop quantal level-$k$ policies for every agent and every $k = 1,\dots,k_{max}$.

If we strictly consider positive rewards and set $x^{0} {=} 0$, we have the CPT value in \eqref{equ:CPT_measure_discrete} reduced to a form that includes only $u^{+}:\mathbb{R}^+{\rightarrow} \mathbb{R}^+$ and $\tilde{\rho}^+:[0,1]{\rightarrow}[0,1]$. In \cite{lin2013stochastic}, it is proved that under such condition, the $\mathbb{CPT}$ measure is a reward transition mapping (Theorem 3.2). Thus, following Section 2 in \cite{lin2013stochastic}, given current state $s_t$, the discounted future cumulative prospects that a risk-sensitive quantal level-$k$ agent $i$ tries to maximize can be expressed as:

\begin{align}\label{equ: V-risk-aware-inifite-RHO}
\max_{\pi^{i,k}}J_{\pi^{i,k}}(s_t){=}\max_{\pi^{i,k}}\mathbb{CPT}_{\pi^{*,-i,k-1}}\Big [R^{i}(s_{t},a^{i}_{t},a^{-i}_{t}) {+} \cdots \nonumber\\
+ \tilde{\gamma}^{\tau} \mathbb{CPT}_{\pi^{*,-i,k-1}}\Big[R^{i}(s_{t+\tau},a^{i}_{t+\tau},a^{-i}_{t+\tau}) + \dots \Big]\Big],
\end{align}
\normalsize
where $\pi^{*,-i,k-1}{:} \mathcal{S}{\times} A^{-i} {\to} [0,1]$ denotes the optimal risk-sensitive quantal level-$(k{-}1)$ policy of agent $-i$ whose level of intelligence is believed to be $(k{-}1)$ from agent $i$'s perspective. The action $a^{-i}_{t+\tau}$ denotes the predicted action of agent $-i$ sampled from $\pi^{*,-i,k-1}$ at time step $t{+}\tau$.

We define $V^{*,i,k}(s_t){\triangleq} J_{\pi^{*,i,k}}(s_t)$ as the optimal CPT value that $i$ could collect following $\pi^{*, i, k}$ starting from $s_t$. Then, the optimal CPT value at any $s{\in}\mathcal{S}$ satisfies \cite{ruszczynski2010risk, lin2013dynamic}:

\begin{align}\label{equ: risk-aware-level-k-dp-intergral}
V^{*,i,k}(s) & = \max_{a^i \in A^i} \mathbb{CPT}_{\pi^{*,-i,k-1}} \big[ R^i(s,a^i,a^{-i}) + \nonumber\\
&\tilde{\gamma}V^{*,i,k}(s')\big], s' = \mathcal{T}_{a^{-i}\sim\pi^{*,-i,k-1}}(s,a^i,a^{-i}).
\end{align}
\normalsize
We also define the optimal CPT Q-value of agent $i$ as $Q^{*,i,k}(s, a^i){=}\mathbb{CPT}_{\pi^{*,-i,k-1}} \big[ R^i(s,a^i,a^{-i}) {+} \tilde{\gamma}V^{*,i,k}(s')\big]$.
Based on the Boltzmann model \cite{von2007theory}, we re-construct $\pi^{*,i,k}$ as

\begin{align}\label{equ: risk-aware-level-k-stochastic-policy}
\pi^{*,i,k}(s,a^i) = \frac{\exp\big(\beta{Q}^{*,i,k}(a^i,s)\big)}{\sum_{a'\in A^i}\exp\big(\beta{Q}^{*,i,k}(a',s)\big)},
\end{align}
\normalsize
where $\beta{\ge}0$ defines the level of the agents conforming to the optimal strategy. Without loss of generality, we set $\beta{=}1$. By iteratively solving (\ref{equ: risk-aware-level-k-dp-intergral}), the optimal quantal level-$k$ risk-sensitive policy $\pi^{*,i,k}$ for every $i{\in}\mathcal{P}$ and every $k {=} 1,\dots,k_{\text{max}}$ can be obtained.

\subsection{Policy Convergence}\label{sec: value iteration}

In classical MGs, $V^{*,i,k}(s)$ in \eqref{equ: risk-aware-level-k-dp-intergral} can be solved via standard value iteration algorithm. Note that the CPT measure in \eqref{equ: risk-aware-level-k-dp-intergral} is non-convex and nonlinear, thus the conditions for the convergence of value iteration algorithm for solving \eqref{equ: risk-aware-level-k-dp-intergral} need to be-established.

\begin{theorem}\label{theorem: value_iteration}
Denote $\langle s, a^i, a^{-i}\rangle{:=}c^{a^{-i}}_{s,a^i}$ and normalize $\tilde{\rho}^i(c^{a^{-i}}_{s,a^i}){:=}\tilde{\rho}^i(\mathbb{P}(a^{-i}|s,a^i))$ by

\begin{equation}\label{equ: probability function transformation}
\hspace{-0.1cm}\rho^i(c^{a^{-i}}_{s,a^i})  = {\tilde{\rho}^i(c^{a^{-i}}_{s,a^i})}/{{\sum}_{a^{-i'}}  \tilde{\rho}^i(c^{a^{-i'}}_{s,a^i})},
\end{equation}
where $\tilde{\rho}$ refers to $\tilde{\rho}^{+}$ defined in \eqref{equ:CPT_measure_discrete} since we consider only positive rewards. For an arbitrary agent $i{\in}\mathcal{P}$, if the one-step reward $R^i$ is lower-bounded by $R_{\text{min}}$ with $R_{\text{min}} \geq 1$, then $\forall s \in \mathcal{S}$ and all intelligence levels with $k{=}1,2,\cdots$, the dynamic programming problem in (\ref{equ: risk-aware-level-k-dp-intergral}) can be solved by the following value iteration algorithm (\cref{alg: policy algorithm}):
\begin{align}\label{equ: risk-aware-level-k-dp-discretized-value-iteration}
V^{i,k}_{m+1}(s) & = \max_{a^i\in A^i} {\sum}_{a^{-i} \in A^{-i}}\rho^{i}(c^{a^{-i}}_{s,a^i}) u^{i}\big(R^i(s,a^i,a^{-i}) {+} \nonumber\\
&\tilde{\gamma}V_m^{i,k}(s')\big),\quad  s' = \mathcal{T}(s,a^i,a^{-i}),
\end{align}
\normalsize
where $u^i$ refers to agent $i$'s instance of $u^{+}$ in \eqref{equ:CPT_measure_discrete}. Moreover, as $m \rightarrow \infty$, $V^{i,k}_{m+1}$ converges to the optimal value function $V^{*,i,k}(s)$.
\end{theorem}

\begin{proof}
Detailed proof is given in Section B of the supplementary material. Here, we show only the skeleton. As shown in \cref{sec: BRSMG policy}, the iterative format of level-$k$ policies has reduced \eqref{equ: risk-aware-level-k-dp-intergral} to a single-agent policy optimization problem with known $\pi^{*,-i,k{-}1}$ from previous iterations. Hence, we only need to show that the CPT operator defined by $\mathcal{B}V^{i,k}_{m} = V^{i,k}_{m+1}$ is a contraction when $R_{\min}{\ge}1$ for any $k{\ge}1$ (Lemma 2 in Section B of the supplementary material).
\end{proof}

\begin{algorithm}[t]
    \footnotesize
    \caption{\small Risk-sensitive quanntal level-$k$ policies}
    \label{alg: policy algorithm}
    
    \textbf{Input}: Markov Game $\mathcal{G}$, $k_{\text{max}}$, and the anchoring policy $\pi^{0}$.
    
    \textbf{Output}: $\{\pi^{*,i,k}\}$, $i\in\mathcal{P}$ and $k =1,\dots,k_{\text{max}}$.
    
    \For{$k = 1:k_{\text{max}}$}{
        \For{$i \in \mathcal{P}$}{
            Initialize $V^{i,k}(s), \forall s \in \mathcal{S}$;
            
            \While{$V^{i,k}$ not converged}{
                
                \For{$s\in\mathcal{S}$}{
                    $V^{i,k}(s) \leftarrow \mathcal{B} V^{i,k}(s)$;
                }
            }
            \For{$(s,a^i) \in \mathcal{S}\times A^i$}{
                Compute $\pi^{*,i,k}(s,a^i)$ based on (\ref{equ: risk-aware-level-k-stochastic-policy});
            }
        }
    }
    Return $\{\pi^{*,i,k}\}$, $i\in\mathcal{P}$ and $k \in \mathbb{K}$.
    \normalsize
\end{algorithm}

\section{The Inverse Reward Learning Problem}\label{sec: IRL}

We now consider the inverse learning problem in BRSMGs. Given demonstrated trajectories of two interacting agents who are running the quantal level-$k$ risk-sensitive policies, our goal is to infer agents' rewards, risk-sensitive parameters, and levels of intelligence.

\subsection{Formulation of the Inverse Learning Problem }

We assume that the one-step rewards for both agents can be linearly parameterized by a group of selected features: $\forall i{\in}\mathcal{P}, R^i(s, a^i, a^{-i}) {=} (\omega^i)^\intercal \Phi^i(s, a^i, a^{-i})$, where $\Phi^i(s,a^i,a^{-i}){: }\mathcal{S} {\times} A^i {\times} A^{-i} {\to} \mathbb{R}^d$ is a known feature function that maps a game state $s$, an action of agent $i$, and an action of agent $-i$ to a $d$-dimensional feature vector, and $\omega^i {\in} \mathbb{R}^d$ is a $d$-dimensional reward parameter vector. We define $\bar{\omega}{=} (\bar{\gamma},\bar{\omega}^{\text{r}}, \bar{k})$, where $\bar{\gamma} {=} (\gamma^{i},\gamma^{-i})$, $\bar{\omega}^{\text{r}}{=}(\omega^i, \omega^{-i})$, and $\bar{k}=(k^{i},k^{-i})$, respectively, represent the parameters in the weighting functions in (1b), the reward parameter vectors, and the levels of intelligence of both agents. Thus, given a set of demonstrated trajectories from the two players in a BRSMG denoted by $\mathcal{D}{=}\{\xi_1, {\cdots}, \xi_M\}$ with $\xi {=} \{(s_0,\bar{a}_0),\dots,(s_{N{-}1},\bar{a}_{N{-}1})\}$, $s_t{\in}\mathcal{S}$, and $\bar{a}_{t}{\in}\mathcal{A}$ ($t{=}0,{\dots},N{-}1$), the inverse problem aims to retrieve the underlying reward parameters, the risk-sensitive parameters, and the levels of intelligence of the agents from $\mathcal{D}$. Based on the principle of Maximum Entropy \cite{ziebart2008maximum}, the problem is equivalent to solving the following optimization problem:

\footnotesize
\begin{align}\label{equ: maximum likelihood}
	\max_{\bar{\omega}} {\sum}_{\xi\in \mathcal{D}} \hspace{-0.3cm}\log \mathbb{P}\left(\xi | \bar{\omega}\right)= \max_{\bar{\omega}} {\sum}_{\xi\in \mathcal{D}} \hspace{-0.3cm}\log {\prod}_{t=0}^{N-1}\mathbb{P}(\bar{a}_t | s_{t},\bar{\omega}),
\end{align}
\normalsize
where $\mathbb{P}(\bar{a}_t|s_t,\bar{\omega})$ is the joint likelihood of agents' actions conditioned on states and parameters and can be expressed as
\small
\begin{align} \label{equ: policy likelihood}
   	\hspace{-0.2cm} \log{\mathbb{P}(\bar{a}_t | s_{t},\bar{\omega})} &= \log \pi^{*,i,k^i}_{(\bar{\gamma}, \bar{\omega}^r)}(s_t,a^i_t)\pi^{*,{-i},k^{-i}}_{(\bar{\gamma}, \bar{\omega}^r)}(s_t,a^{-i}_t),
\end{align}
\normalsize
where $\pi^{*,i,k^i}_{(\bar{\gamma}, \bar{\omega}^r)}$ and $\pi^{*,-i,k^{-i}}_{(\bar{\gamma}, \bar{\omega}^r)}$, respectively, represent the risk-sensitive quantal level-$k$ policies of agent $i$ and agent $-i$ induced by parameters $(\bar{\gamma}, \bar{\omega}^r)$.

\noindent
\textbf{Problem approximation.} The optimization \eqref{equ: maximum likelihood} can be formulated as a mixed-integer optimization which is infeasible to solve. Therefore, we make the following approximation: we remove $\bar{k}$ from $\bar{\omega}$, and treat $\bar{k}$ as representations of agents' internal states which can be inferred based on agents' demonstrations and most recent estimates of their reward parameters and risk-measure parameters. With that, we evaluate the expected likelihood of $\bar{a}_t$ with respect to the inferred distributions of $\bar{k}$, and solve \eqref{equ: maximum likelihood} via gradient ascent.

\subsection{The Gradient Information}\label{sec: likelihoood_demonstration}

With the proposed approximation described above, we re-define $\bar{\omega}$ as $(\bar{\gamma},\bar{\omega}^{\text{r}})$, then \eqref{equ: policy likelihood} can be re-written as:
\begin{align} \label{equ: trajectory likelihood full}
   	& \log \mathbb{E}_{\bar{k}|\xi_{t-1},\bar{\omega}} \Big[{\mathbb{P}(\bar{a}_t | s_{t},\bar{\omega})}\Big]
	{=} \log{\sum}_{(k^i,k^{-i}) \in \mathbb{K}^2} \pi^{*,i,k^i}_{\bar{\omega}}(s_t,a^i_t)\nonumber\\
    & \cdot\pi^{*,{-i},k^{-i}}_{\bar{\omega}}(s_t,a^{-i}_t)\mathbb{P}(k^i | \xi_{t-1}, \bar{\omega} ) \mathbb{P}(k^{-i} | \xi_{t-1}, \bar{\omega} ),
\end{align}
where $\mathbb{P}(k^i | \xi_{t-1}, \bar{\omega} )$, $k\in\mathbb{K}$, is the posterior belief in an agent's intelligence level inferred based on the joint trajectory history upon time $t{-}1$. Note that initially, we set $\mathbb{P}(k^i|\xi_{-1}, \bar{\omega} )$ as an uniform distribution over $\mathbb{K}$. Then, $\mathbb{P}(k^i | \xi_{t-1}, \bar{\omega} )$ can be updated recursively from $t=0$ using Bayesian inference:

\begin{align} \label{equ: model identification}
\mathbb{P}(k^i | \xi_t, \bar{\omega})= \frac{\pi^{*,i,k^i}_{\bar{\omega}}(s_t,a_t^i) \mathbb{P}(k^i | \xi_{t-1}, \bar{\omega})}{\sum_{k' \in \mathbb{K}} \pi^{*,i,k'}_{\bar{\omega}}(s_t,a_t^{i},) \mathbb{P}(k' | \xi_{t-1}, \bar{\omega})}.
\end{align}
\normalsize

From \eqref{equ: maximum likelihood}, \eqref{equ: trajectory likelihood full} and \eqref{equ: model identification}, we can see that the gradient ${ \partial \log\mathbb{P}(\xi|\bar{\omega})}/{\partial \bar{\omega}}$ depends on two items (details are in Section C of the supplementary material): 1) the gradients of both agents' policies under arbitrary intelligence level $k{\in}\mathbb{K}$ with respect to $\bar{\omega}$, i.e., ${\partial \pi^{*,i,k}_{\bar{\omega}}}/{\partial \bar{\omega}}$ and 2) the gradients of posterior beliefs in agents' intelligence levels with respect to $\bar{\omega}$, i.e., $\partial \log\mathbb{P}(k^i|\xi_{t{-}1},\bar{\omega})/\partial \bar{\omega}$.

\noindent
\textbf{Gradients of policies.} Recall \eqref{equ: risk-aware-level-k-stochastic-policy}, ${\partial \pi^{*,i,k}_{\bar{\omega}}}/{\partial \bar{\omega}}$, $\forall i{\in}\mathcal{P}$ and $k{\in}\mathbb{K}$, requires the gradient of the corresponding optimal $Q$ function with respect to $\bar{\omega}$, i.e., $\partial Q_{\bar{\omega}}^{*,i,k}/\partial\bar{\omega}$ (detailed derivation is shown in Section C.1 of the supplementary material). Due to the $\max$ operator in \eqref{equ: risk-aware-level-k-dp-intergral}, direct differentiation is not feasible. Hence, we use a smooth approximation for the $\max$ function, that is, $\max(x_1,{\cdots},x_{n_x}) {\approx} \big(\sum_{i=1}^{n_x} (x_i)^\kappa\big)^{\frac{1}{\kappa}}$ with all $x_i{>}0$. The parameter $\kappa {>} 0$ controls the approximation error, and when $\kappa {\to} \infty$, the approximation becomes exact. Therefore, \eqref{equ: risk-aware-level-k-dp-intergral} can be re-written as
\begin{align}\label{equ: risk-aware-level-k-V-p_norm}
V_{\bar{\omega}}^{*,i,k}(s) & ={\max}_{a^i \in A^i} Q_{\bar{\omega}}^{*,i,k}(s,a^i)\nonumber\\
& \approx \Bigg( {\sum}_{a^i \in A^i} \Big( Q_{\bar{\omega}}^{*,i,k}(s,a^i)\Big)^{\kappa} \Bigg)^{\frac{1}{\kappa}}.
\end{align}
\normalsize
Taking derivative of both sides of \eqref{equ: risk-aware-level-k-V-p_norm} with respect to $\bar{\omega}$ yields (note that ${(\cdot)}^{'}_{\bar{\omega}} := \frac{\partial (\cdot)_{\bar{\omega}}}{\partial \bar{\omega}}$):
\footnotesize
\begin{subequations} \label{equ: Value-gradient-dp}
\begin{align}
    & V^{',*,i,k}_{\bar{\omega}}(s)
    {\approx}  \frac{1}{\kappa} \Bigg( {\sum}_{a^i \in A^i} \Big( Q^{*,i,k}_{\bar{\omega}}(s,a^i)\Big)^{\kappa} \Bigg)^{\frac{1{-}\kappa}{\kappa}}\\
     &\cdot{\sum}_{a^i{\in} A^i} \Big[ \kappa \Big( Q^{*,i,k}_{\bar{\omega}}(s,a^i)\Big)^{\kappa{-}1} \cdot Q^{',*,i,k}_{\bar{\omega}}(s,a^i)\Big],\nonumber\\
    & Q^{',*,i,k}_{\bar{\omega}} (s,a^i) = {\sum}_{a^{-i}\in A^{-i}} \Bigg( \frac{\partial \rho^{i}_{\bar{\omega}}}{\partial \bar{\omega}}(c^{a^{-i}}_{s,a^i}) u^{i}\big(R^i_{\bar{\omega}}(s,a^i,a^{-i})\nonumber\\
    & + \tilde{\gamma}V^{*,i,\text{k}}_{\bar{\omega}}(s') \big)  +\rho^{i}_{\bar{\omega}}(c^{a^{-i}}_{s,a^i})\alpha\big(R^i_{\bar{\omega}}(s,a^i,a^{-i}) \\
    & + \tilde{\gamma}V^{*,i,\text{k}}_{\bar{\omega}}(s')\big)^{\alpha-1}\big(\frac{\partial R^i_{\bar{\omega}}}{\partial \bar{\omega}}(s,a^i,a^{-i}) + \tilde{\gamma}V^{',*,i,k}_{\bar{\omega}}(s')\big)\Bigg).\nonumber
\end{align}
\end{subequations}
\normalsize

Notice that in \eqref{equ: Value-gradient-dp}, $V^{',*,i,k}_{\bar{\omega}}$ is in a recursive format. Hence, we propose below a dynamic programming algorithm to solve for $V^{',*,i,k}_{\bar{\omega}}$ and $Q^{',*,i,k}_{\bar{\omega}}$ at all state and action pairs.

\begin{theorem}\label{theorem: value-gradient-iteration}
If the one-step reward $R^i$, $i\in\mathcal{P}$, is bounded by $R^i{\in}[R_{\text{min}}, R_{\text{max}}]$ satisfying $\frac{R_{\text{max}}}{R_{\text{min}}^{2-\alpha}}\alpha\tilde{\gamma}  {< }1$, then $\partial V_{\bar{\omega}}^{*,i,k}/\partial\bar{\omega}$ can be found via the following value gradient iteration:
\footnotesize
\begin{subequations}
	\begin{align} \label{equ: V_gradient_iteration}
		&  V^{',i,k}_{\bar{\omega},m+1} (s) \approx  \frac{1}{\kappa} \Bigg( {\sum}_{a^i \in A^i} \Big( Q^{*,i,k}_{\bar{\omega}}(s,a^i)\Big)^{\kappa} \Bigg)^{\frac{1-\kappa}{\kappa}}\\
		& \cdot{\sum}_{a^i\in A^i} \Bigg[ \kappa \Big( Q^{*,i,k}_{\bar{\omega}}(s,a^i)\Big)^{\kappa-1} \cdot  Q^{',i,k}_{\bar{\omega},m} (s,a^i)\Bigg],\nonumber\\
		&  Q^{',i,k}_{\bar{\omega},m} (s,a^i) = {\sum}_{a^{-i}\in A^{-i}} \Bigg( \frac{\partial \rho^{i}_{\bar{\omega}}}{\partial \bar{\omega}}(c^{a^{-i}}_{s,a^i})u^i\big(R^i(s,a^i,a^{-i}) \nonumber\\
		& +\tilde{\gamma}V^{*,i,k}_{\bar{\omega}}(s')\big) +\rho^{i}_{\bar{\omega}}(c^{a^{-i}}_{s,a^i}) \alpha\big(R^i_{\bar{\omega}}(s,a^i,a^{-i})\nonumber\\
		& + \tilde{\gamma}V^{*,i,\text{k}}_{\bar{\omega}}(s')\big)^{\alpha-1}\big(\frac{\partial R^i_{\bar{\omega}}}{\partial \bar{\omega}}(s,a^i,a^{-i}) + \tilde{\gamma}V^{',i,k}_{\bar{\omega},m}(s')\big)\Bigg).
	\end{align}
\end{subequations}
\normalsize
Moreover, the algorithm converges to $\partial V_{\bar{\omega}}^{*,i,k}/\partial\bar{\omega}$ as $m {\to} \infty$.
\end{theorem}
\begin{proof}
We first define $\nabla\mathcal{B}V^{',i,k}_{m} = V^{',i,k}_{m+1}$, and show that the operator $\nabla\mathcal{B}$ is a contraction under the given conditions (derivations of ${\partial \rho^{i}_{\bar{\omega}}}/{\partial \bar{\omega}}$ are shown in Section C.2 of the supplementary material). Then, the statement is proved by induction similar to \cref{theorem: value_iteration}. More details are given in Section D of the supplementary material.
\end{proof}

\begin{algorithm}
    \footnotesize
    \caption{Gradients of quantal level-$k$ risk-sensitive policies}
    \label{alg: policy gradient algorithm}
    
    \textbf{Input}: Markov Game model $\mathcal{G}$, highest intelligence level $k_{\text{max}}$, and $\pi^{*,i,k}$, $i\in\mathcal{P}$ and $k = 1,\dots,k_{\text{max}}$.
    
    \textbf{Output}:$\{ \frac{\partial \pi^{*,i,k}_{\bar{\omega}}}{\partial \bar{\omega}} \}$, $i\in\mathcal{P}$ and $k \in \mathbb{K}$.
    
    \For{$k = 1,\dots,k_{\text{max}}$}{
        \For{$i \in \mathcal{P}$}{
            Initialize $V^{',i,k}_{\bar{\omega}}(s), \forall s \in \mathcal{S}$;
            
            \While{$V^{',i,k}$ not converged}{
                
                \For{$s\in\mathcal{S}$}{
                    $V^{',i,k}(s) \leftarrow \nabla \mathcal{B} V^{',i,k}(s)$;
                }
            }
            \For{$(s,a^i) \in \mathcal{S}\times A^i$}{
                Compute $ \frac{\partial \pi^{i,k}_{\bar{\omega}}}{\partial \bar{\omega}} (s,a^i) $ by differentiating \cref{equ: risk-aware-level-k-stochastic-policy} with respect to $\omega$;
            }
        }
    }
    
    Return $\{ \frac{\partial \pi^{i,k}_{\bar{\omega}}}{\partial \bar{\omega}} \}$, $i\in\mathcal{P}$ and $k \in \mathbb{K}$.
\normalsize
\vspace{-0.1cm}
\end{algorithm}

\noindent
\textbf{Gradient of the posterior belief.} {We summarize the value iteration algorithm that computes the policy gradient in \cref{alg: policy gradient algorithm}.} The second gradient that we need to compute is the gradient of the posterior belief in $k$ with respect to $\bar{\omega}$, i.e., $\partial \log\mathbb{P}(k^i|\xi_{t{-}1},\bar{\omega}) /\partial \bar{\omega}$. Recall \eqref{equ: model identification}, we have $\partial \log\mathbb{P}(k^i|\xi_{t{-}1},\bar{\omega})/\partial \bar{\omega}$ depending on $\partial \pi_{\bar{\omega}}^{*,i,k} / \partial \bar{\omega}(s_{t{-}1}, a^i_{t{-}1})$ and $\partial \log\mathbb{P}(k|\xi_{t{-}2},\bar{\omega})/\partial \bar{\omega}$ for all $k{\in}\mathbb{K}$. Substituting the gradients of policies obtained through \cref{alg: policy gradient algorithm} in $\partial \log\mathbb{P}(k|\xi_{t{-}1},\bar{\omega})/\partial \bar{\omega}$ yields a recursive format from time $0$ to time $t{-}1$, which can be easily computed.

\noindent
\textbf{Generalization to other iterative reasoning models.} Both \cref{theorem: value_iteration} and \cref{theorem: value-gradient-iteration} naturally extend to other probabilistic iterative reasoning models as long as the optimal policies are iterative and satisfy \eqref{equ: risk-aware-level-k-dp-intergral}. For instance, the quantal cognitive hierarchy model \cite{wright2014level} that allows for mixed levels of intelligence can be well applied. Detailed extension and comparison among these models are left to future work.

\subsection{The Inverse Learning Algorithm in BRSMG}
With the gradient of \eqref{equ: maximum likelihood} defined, the gradient ascent algorithm is used to find local optimal parameters in $\bar{\omega}$ that maximize the log-likelihood of demonstrations in a BRSMG \cref{alg: learning algorithm}.

\begin{algorithm}
	\caption{The inverse learning algorithm}
	\label{alg: learning algorithm}
	\footnotesize
	\textbf{Input:} A demonstration set $\mathcal{D}$ and learning rate $\eta$
	
	\textbf{Output:} Learned parameters $\bar{\omega}$.
	
	Initialize $\bar{\omega}$.
	
	\While{\text{not converged}}{
		
		Run \cref{alg: policy algorithm}, \cref{alg: policy gradient algorithm}
		
		Compute gradient of the log-likelihood of the demonstration following: $\nabla_{\bar{\omega}} = \sum_{\xi\in\mathcal{D}}\frac{ \partial \log\big(\mathbb{P}(\xi|\bar{\omega})\big)}{\partial \bar{\omega}}$;
		
		Update the parameters following: $\bar{\omega} = \bar{\omega} + \eta \nabla_{\bar{\omega}}$;
	}
	
	\textbf{Return:} $\bar{\omega}$
\end{algorithm}
\normalsize

\section{Experiments}\label{sec: simulation}

In this section, we utilize a grid-world navigation example to verify the proposed algorithms in both the forward policy design problem and the inverse reward learning problem in a BRSMG. The simulation setup is shown in \cref{fig: intro_figure}. Two human agents must exit the room through two different doors while avoiding the obstacles and potential collisions with each other. We assume that the two agents move simultaneously, and they can observe the actions and states of each other in the previous time step. Moreover, we let $k_{\text{max}}{=}2$ in this experiment since psychology studies found that most humans perform at most two layers of strategic thinking \cite{stahl1995players}. 

\subsection{Environment Setup}
We define the state as $s {=} (x^1,y^1,x^2,y^2)$, where $x^i$ and $y^i$ denote the coordinates of the human agent $i$, $i{\in}\mathcal{P}$. The two agents share a same action set $A {=} \{ \text{move left}, \text{move right}, \text{move up}, \text{move down}, \text{stay}\}$. In each state, the reward of agent $i$ includes two elements: a navigation reward as shown in \cref{fig: nominal reward} and a safety reward that reflects the penalty for collisions. We restrict all rewards to be positive, satisfying $R_{\min}{=}1$ and $\frac{R_{\text{max}}}{R_{\text{min}}^{2-\alpha}}\alpha\tilde{\gamma}  {< }1$. If a collision happens, an agent will collect a fixed reward of $1$. If there is no collision, agents receive rewards greater than $1$ according to the navigation reward map.

\noindent\textbf{Selection of the quanntal level-$0$ policy}. Recall that a quantal level-$0$ policy is required to initiate the iterative reasoning process in \cref{alg: policy algorithm}. In this work, we use an uncertain-following policy as an exemplary quantal level-$0$ policy: from a quantal level-$1$ agent's perspective, a quantal level-$0$ agent is a follower who accommodates the quantal level-$1$ agent's planned immediate action. Namely, given state $s_t$ and action $a^{-i}$ from the opponent agent (i.e., the leader), the stochastic policy of a level-$0$ agent $i$ satisfies

\small
\begin{equation}\label{equ: anchor policy}
\pi^{*,i,0} (s_t, a^{i}| a^{-i}) = \frac{\exp\big({R^{i}(s_t,a^{i},a^{-i})}\big)}{\sum_{a'\in A^{i}} \exp\big(R^{i}(s_t,a',a^{-i})\big)},  \forall a^{i}\in A^{i}.    
\end{equation}
\normalsize

\begin{figure}
\begin{center}
\begin{picture}(450, 90)
\put(20,  0){\epsfig{file=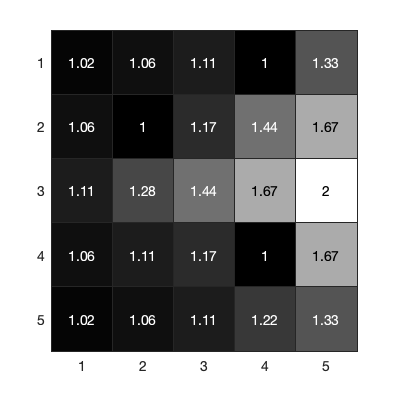,width=0.4 \linewidth, trim=1.2cm 1.2cm 0.cm 0cm,clip}}
\put(130,  0){\epsfig{file=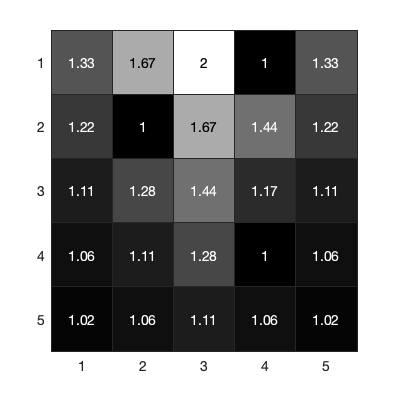,width=0.398 \linewidth, trim=1.2cm 1.2cm 0.cm 0cm,clip}}
\end{picture}    
\end{center}
\caption{ The navigation reward maps satisfying $R{\ge}1$ (left: the orange agent; right: the blue agent).}\label{fig: nominal reward}
\end{figure}

\subsection{Interactions in BRSMG}
\label{subsec: simu_forward}
In this section, we investigate the influence of the risk-sensitive performance measure on agents' policies in a Markov Game by comparing agents' interactive behaviors under risk-neutral and risk-sensitive policies. We set the parameters in the CPT model as $\gamma^{1,2} {=} 0.5$ and $\alpha^{1,2} {=} 0.7$. 

Three cases are considered: Case 1 - both agents are quantal level-$1$ (L1-L1); Case 2 - both agents are quantal level-$2$ (L2-L2); and Case 3 - one agent is quantal level-$1$ and the other is quantal level-$2$ (L1-L2). If both agents exit the environment without collisions and dead-locks, we call it a {success}. We compare the rate of success (RS) of each case under risk-neutral and risk-sensitive policies in 100 simulations with agents starting from different locations. 

\noindent\textbf{Impacts of bounded intelligence}. First, let us see how a risk-neutral agent behaves under different levels of intelligence. Based on the selected anchoring policy in \eqref{equ: anchor policy}, a risk-neutral quantal level-$1$ agent will behave quite aggressively since it believes that the other agent is an uncertain-follower. On the contrary, a risk-neutral quantal level-$2$ agent will perform more conservatively because it believes that the other agent is aggressively executing a quantal level-$1$ policy. \cref{fig: result_1}(b) shows an exemplary trajectory of Case 1. We can see that with two level-$1$ agents, collision happened due to their aggressiveness, i.e., they both assumed that the other would yield. On the other hand, \cref{fig: result_1}(d) and \cref{fig: result_1}(f), respectively, show exemplary trajectories of Case 2 and Case 3 with agents starting from the same locations as in the exemplary trajectory in \cref{fig: result_1}(b). We can see that in both cases, the two agents managed to avoid collisions. In Case 2, both agents behaved more conservatively, and lead to low efficiency (\cref{fig: result_1}(d)), while in Case 3, both agents behaved as their opponents expected and generated the most efficient and safe trajectories (\cref{fig: result_1}(f)). To show the statistical results, we conducted 100 simulations for each case with randomized initial states, and the RS is shown in \cref{fig: result_1}(a) (green). It is shown that similar to what we have observed in the exemplary trajectories, Case 1 lead to the lowest RS, and Case 3 achieved the highest RS. The RS in Case 2 is in the middle because though both agents behaved conservatively, the wrong belief over the other's model may still lead to lower RS compared to Case 3.

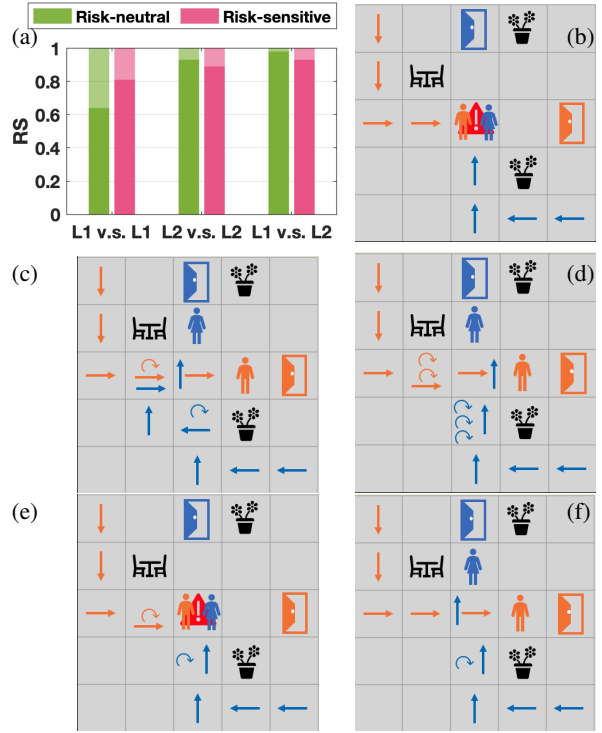
\begin{figure}[h]
	\begin{center}
		\begin{picture}(600, 280)
 		\put( 0,  185){\epsfig{file=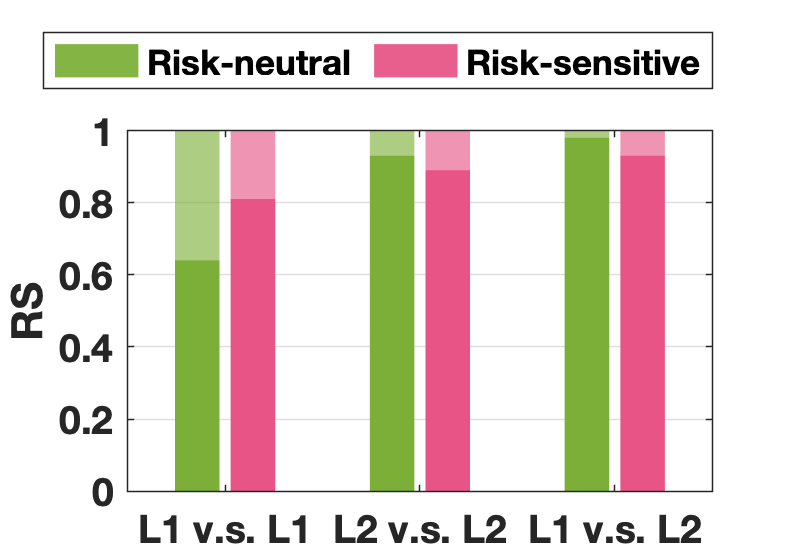,width = 0.55 \linewidth,  trim=0.1cm 0.0cm 0.3cm 0.3cm,clip}}  
		\put( 130,  185){\epsfig{file=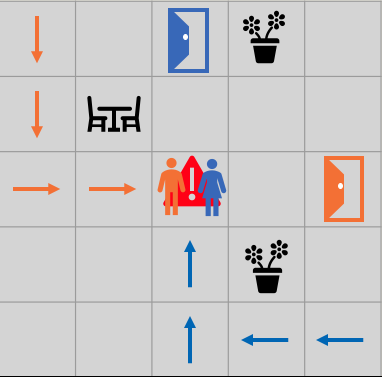,width = 0.38 \linewidth, trim=0.0cm 0.0cm 0.0cm 0cm,clip}}  
 		
 		\put( 25,  90){\epsfig{file=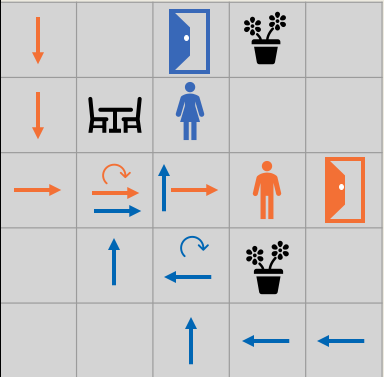,width = 0.38 \linewidth, trim=0.0cm 0.0cm 0.0cm 0cm,clip}}  
 		\put( 130,  90){\epsfig{file=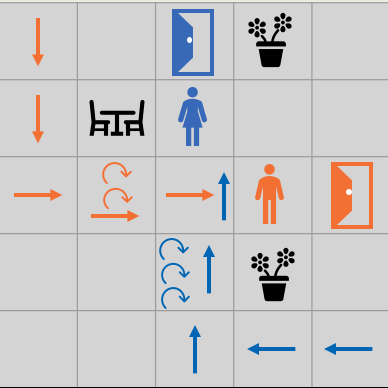,width = 0.38 \linewidth, trim=0.0cm 0.0cm 0.0cm 0cm,clip}}  
		\put(25,  0){\epsfig{file=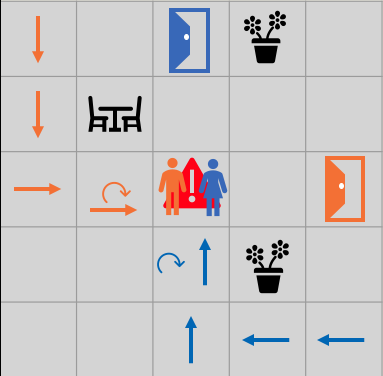,width = 0.38 \linewidth, trim=0.0cm 0.0cm 0.0cm 0cm,clip}}  
		\put( 130,  0){\epsfig{file=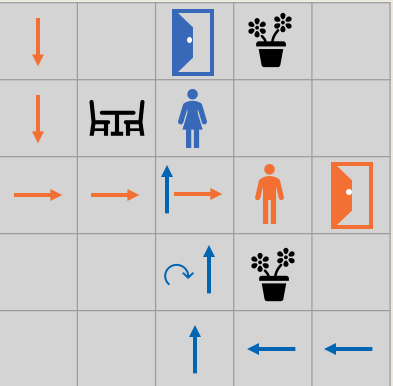,width = 0.38 \linewidth, trim=0.0cm 0.0cm 0.0cm 0cm,clip}}  
		\small
		\put(0,260){(a)}
		\put(210,260){(b)}
		\put(0,170){(c)}
		\put(210,170){(d)}
		\put(0,80){(e)}
		\put(210,80){(f)}
		\normalsize
		\end{picture}
	\end{center}
	\caption{ (a) Performance comparison between the bounded risk-neutral policies and the bounded risk-sensitive policies. (b-f) Examples of interactive trajectories ( circular arrow denotes the action ``stay"); (b) two risk-neutral quantal level-$1$ agents; (c) two risk-sensitive quantal level-$1$ agents; (d) two risk-neutral quantal level-$2$ agents; (e) two risk-sensitive quantal level-$2$ agents; (f) orange: risk-neutral quantal level-$1$ agent; blue: risk-neutral quantal level-$2$ agent.}
	\label{fig: result_1}
\end{figure}

\noindent\textbf{Impacts of risk sensitivity}. In this experiment, we will see how the risk-sensitive CPT model impacts risk-neutral behaviors. As shown in \cref{fig: result_1}(a), in Case 1, the risk-sensitive policies help significantly improve the RS of interactions between two quantal level-$1$ agents. This is because the CPT model makes the quantal level-$1$ agents underestimate the possibilities of ``yielding'' from their opponents, leading to more conservative behaviors with higher RS. Such a conclusion can be verified by comparing the exemplary trajectories shown in \cref{fig: result_1}(b-e). We can see that compared to the risk-neutral case in \cref{fig: result_1}(b), under the risk-sensitive policy, the blue agent decided to yield to the orange one at the fourth step. At the same time, in Case 2 and Case 3, the CPT model makes the quantal level-$2$ agents overestimate the possibilities of ``yielding'' from quantal level-$1$ agents, leading to more aggressive behaviors. An exemplary trajectory is shown in \cref{fig: result_1}(e). We can see that compared to the risk-neutral quantal level-$2$ agents in \cref{fig: result_1}(d), the risk-sensitive quantal level-$2$ agents waited {for} less steps and collide with each other. Hence, the RS for both Case 2 and Case 3 are reduced compared to the risk-neutral scenarios, as shown in \cref{fig: result_1}(a).

\subsection{Reward Learning in BRSMG}
\label{subsec: reward_learning}

In this section, we validate \cref{alg: learning algorithm}. In the inverse problem, we aim to learn the navigation rewards and the CPT parameter $\gamma$ of both agents, (i.e., $\bar{\omega}{=}(\gamma, (\omega^1, \omega^2))$ and $\omega^{1,2}\in\mathbb{R}^{25}$), without prior information on their intelligence levels (we need to infer the intelligence levels simultaneously during the learning). 

\noindent
\textbf{Collecting synthetic expert demonstrations.} We first collect some expert demonstrations in the navigation environment via the policies derived in the forward problem in \cref{sec: BRSMG}. Similarly, for generating the demonstrations, we set the parameters of the CPT model as $\gamma^{1,2} {=} 0.5$ and $\alpha^{1,2} {=} 0.7$, and let agents with mixed intelligence levels interact with each other using the risk-sensitive quantal level-$k$ policies. We randomized the initial conditions {(initial positions and intelligence levels) of the agents} and collected $M{=}100$ expert demonstrations (i.e., paired navigation trajectories). The approximation parameter $\kappa$ in $Q$-value approximation \eqref{equ: risk-aware-level-k-V-p_norm} is set to $\kappa = 100$ and the learning rate is set to $\eta = 0.0015$. 

\noindent
\textbf{Metrics.} We evaluate the learning performance via two metrics: the parameter percentage error (PPE), and the policy loss (PL). The PPE of learned parameters $\bar{\omega}^i$ is defined as $\vert\bar{\omega}^i {- }\bar{\omega}^{*,i}\vert/\vert\bar{\omega}^{*,i}\vert$ with $\bar{\omega}^{*,i}$ being the ground-truth value. The PL denotes the error between the ground truth quantal level-$k$ policies and the policies obtained using the learned reward functions. It is defined as $\frac{1}{|\mathbb{K}\times \mathcal{S}\times A^i|}\sum_{(k,s,a^i)\in \mathbb{K}\times \mathcal{S}\times A^i} |\pi^{*,i,k}_{\bar{\omega}}(s,a) - \pi^{*,i,k}_{\bar{\omega}^*}(s,a)|$ where $\pi^{*,i,k}_{\bar{\omega}}$ and $\pi^{*,i,k}_{\bar{\omega}^*}$ are, respectively, the quantal level-$k$ policy of agent $i$ under the learned parameter vector $\bar{\omega}$ and the true vector $\bar{\omega}^*$.

\noindent
\textbf{Results.} \cref{fig: result_2}(a) and \cref{fig: result_2}(b) show, respectively, the histories of PPE and PL during learning. The solid lines represent the means from $25$ trials, and the shaded areas are the 95\% confidence intervals. The average errors of each learned parameter are given in \cref{fig: result_2}(c). We can see that the proposed inverse learning algorithm can effectively recover both agents' rewards and risk-measure parameter $\gamma$. In addition, in \cref{fig: MEIRL_comparesion}(a), we show the identification accuracy of the intelligence levels of agents in the data. More specifically, the identified intelligence level of agent $i$, $i\in\mathcal{P}$, in a demonstration $\xi$ is given by $\hat{k}_i{=}\argmax_{k\in\mathcal{K}}\mathbb{P}(k|\xi_{N-1})$. We can see that accuracy ratios of 86\% and 92\% are achieved for the two agents, respectively. Overall, the results show that the proposed inverse reward learning algorithm can effectively recover rewards, risk-parameters, and intelligence levels of agents in a BRSMG.

\begin{figure}[h]
	\begin{center}
		\begin{picture}(450, 210)
		\put( 20, 145 ){\epsfig{file=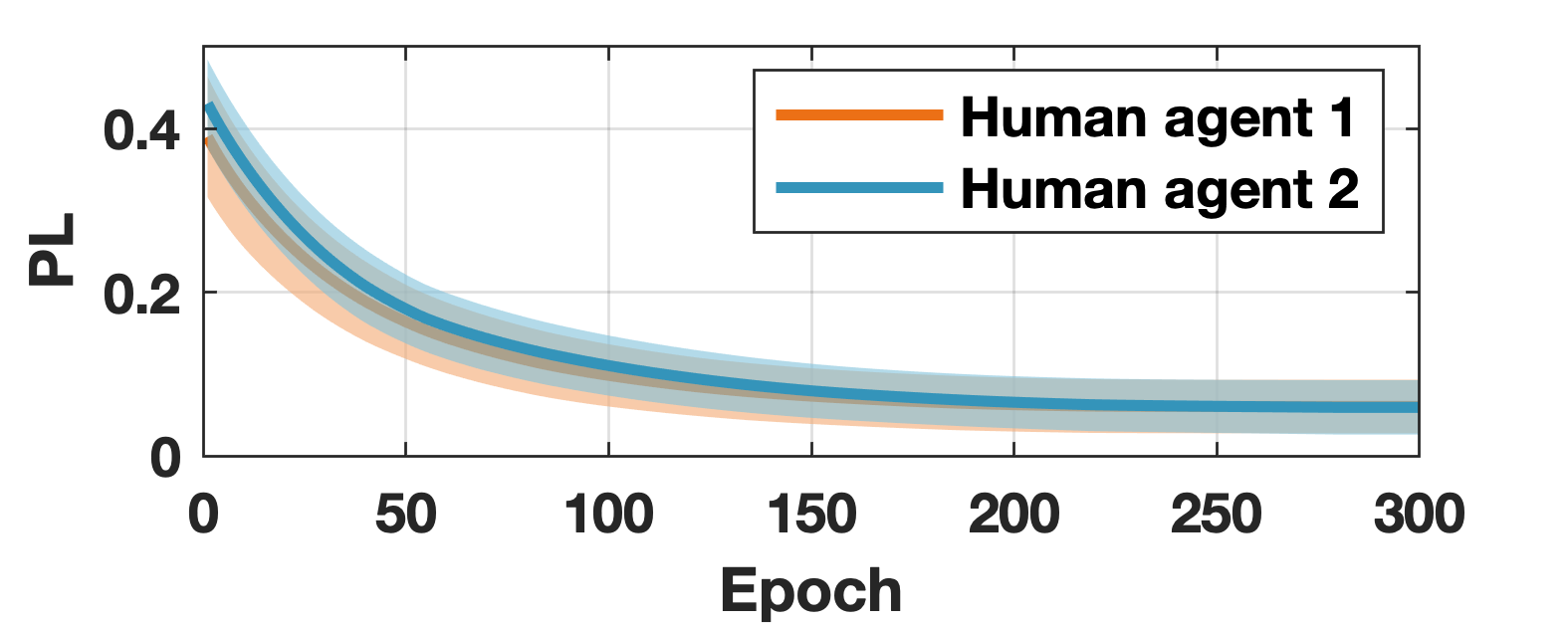,width = 0.7 \linewidth, trim=0.cm 0.0cm 0.5cm 0cm,clip}}
		\put( 20,  75){\epsfig{file=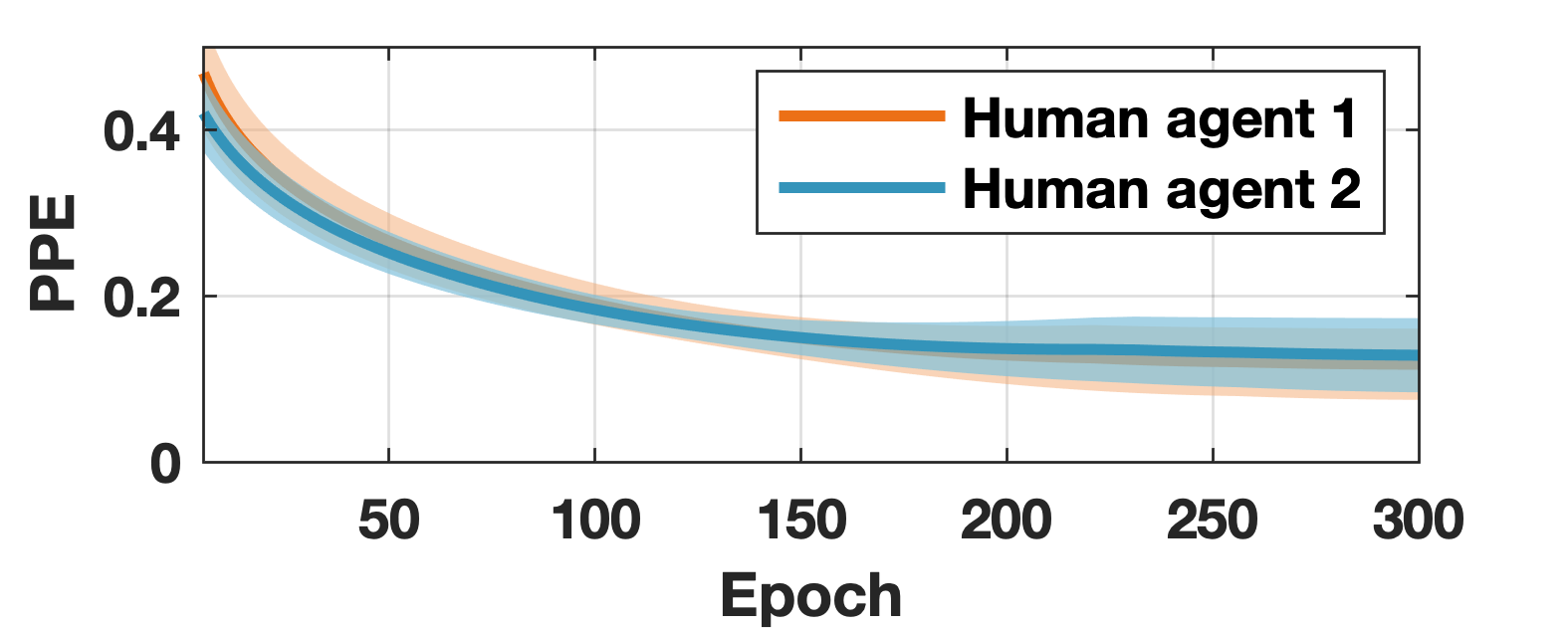,width=0.7\linewidth, trim=0.0cm 0.cm 0.5cm 0cm,clip}} 
		
		\put( 20,  0){\epsfig{file=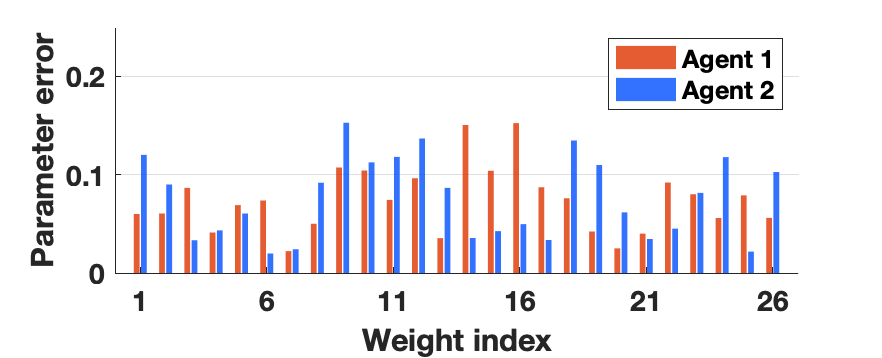,width = 0.75 \linewidth, trim=0.cm 0.0cm 0.5cm 0cm,clip}}  
		\small
		\put(30,210){(a)}
		\put(30,140){(b)}
		\put(30,65){(c)}
		\normalsize
		\end{picture}
	\end{center}
	\caption{(a-b) Averaged PL and PPE w.r.t. training epochs. (c) Average errors of each learned parameter.}
	\label{fig: result_2}
\end{figure}
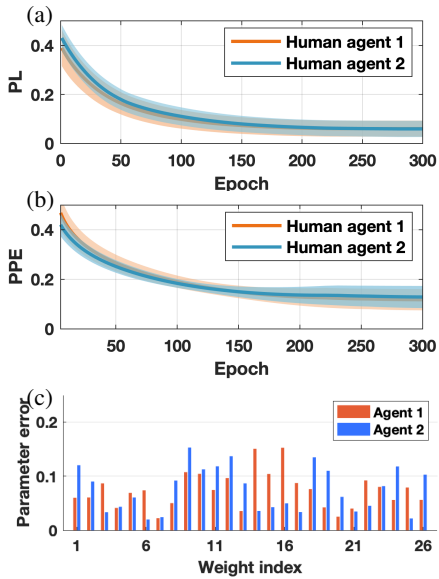

\subsection{Performance Comparison with a Baseline}
In this section, we compare the performance of the proposed inverse reward learning algorithm (BRSMG-IRL) against a baseline inverse reward learning algorithm.

\noindent
\textbf{The baseline IRL algorithm.} The baseline IRL algorithm is a risk-neutral Maximum Entropy IRL algorithm (ME-IRL) without consideration to bounded intelligence, similar to the approach in \cite{sadigh2016planning, sun2018courteous, sun_probabilistic_2018}. Rather than jointly learning both agents' rewards, the baseline runs Maximum Entropy IRL from each agent's perspective separately. In each ego agent' IRL formulation, the interaction is formulated as an open-loop leader-follower game in which the opponent's ground truth trajectory is assumed to be known, making the ego agent a follower to its opponent during learning.

\begin{figure}[ht]
\begin{center}
\begin{picture}(300, 80)

\put(-5,0){\epsfig{file=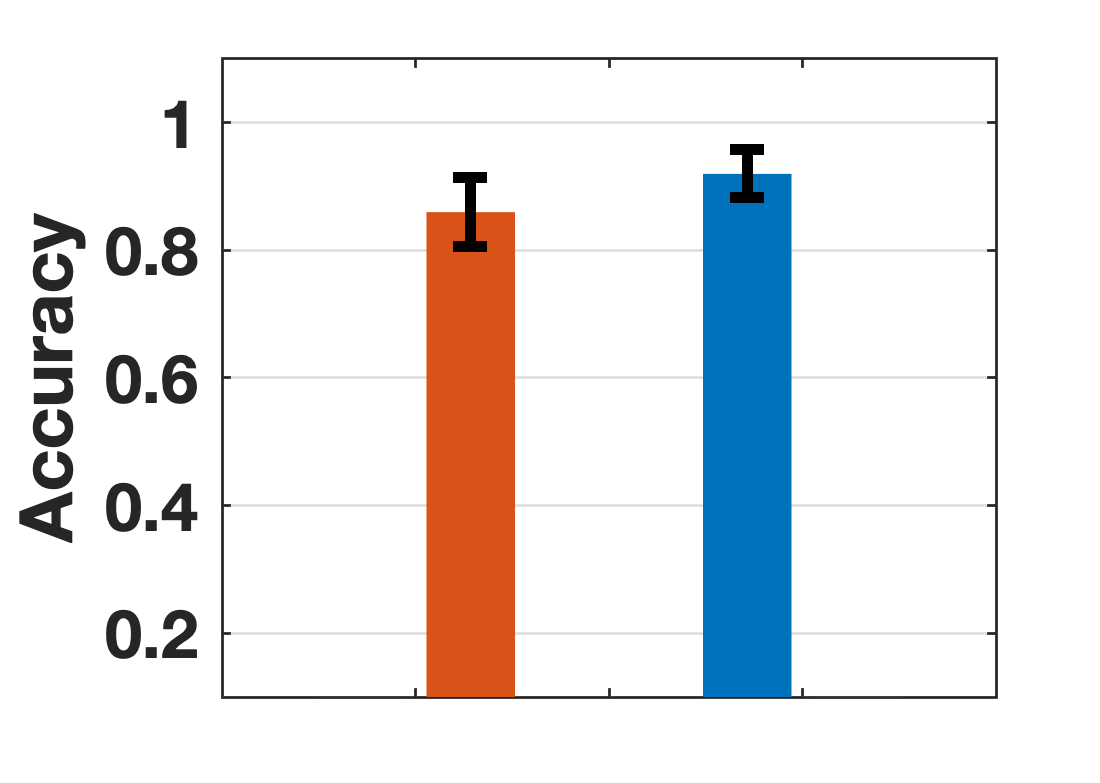,width = 0.45 \linewidth, trim=0.cm 0.0cm 0.5cm 0.5cm,clip}}  

\put(90, 0){\epsfig{file=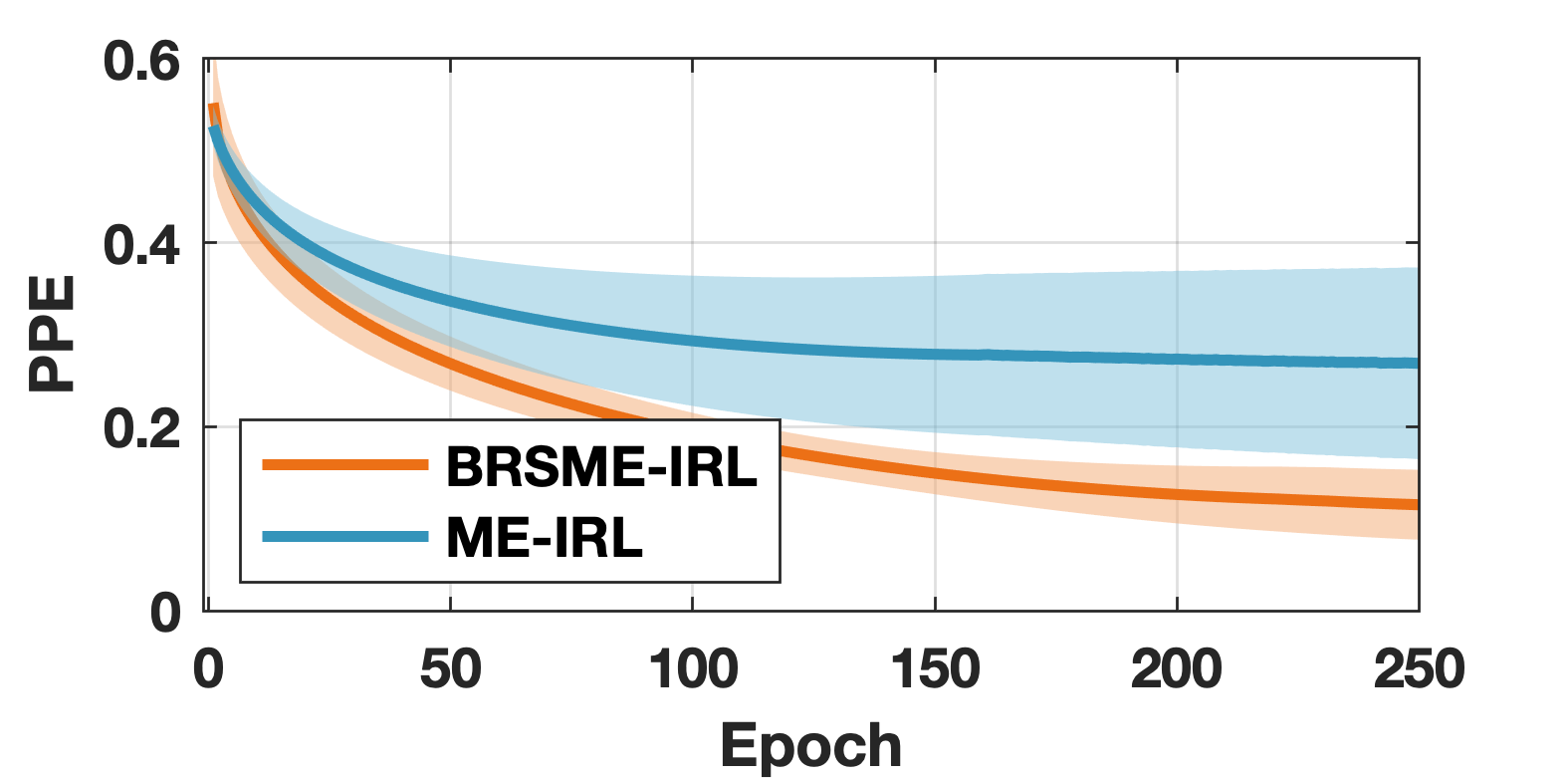,width=0.65 \linewidth, trim=0.0cm 0.0cm 0.5cm 0cm,clip}}

\small
\put(20,70){(a)}
\put(115,70){(b)}
\normalsize
\end{picture}
\end{center}
\caption{ (a): Intelligence level identification accuracy (orange: orange agent; blue: blue agent. (b): Reward learning comparison between our method and a baseline Maximum entropy IRL algorithm.}
\label{fig: MEIRL_comparesion}
\end{figure}

\noindent
\textbf{Metrics.} In addition to PPE and PL, we also compare the learned rewards with the ground truth rewards using two types of statistical correlations: 1) Pearson’s correlation coefficient (PCC) and 2) Spearman’s rank correlation coefficient (SCC). PCC characterizes the linear correlation between the ground truth rewards and the recovered rewards (higher PCC represents higher linear correlations). SCC characterizes the strength and direction of the monotonic relationship between the ground truth rewards and the recovered rewards (higher SCC represents stronger monotonic relationships).

\noindent
\textbf{Results.} The performance comparison between the proposed approach and the baseline is shown in the right plot of \cref{fig: MEIRL_comparesion}. We can see that the proposed method can recover more accurate reward values compared to the baseline. This is because the baseline fails to capture the structural biases caused by agents' risk sensitivity and bounded intelligence. Moreover, \cref{tab: correlation} indicates that the reward values recovered by the proposed method have a higher linear correlation and stronger monotonic relationship to the ground-truth reward values. 

\begin{table}[t]
\centering
\begin{tabular}{c|c c}
\toprule
Algorithm & ME-IRL & BRSMG-IRL\\
\hline
SCC A1 & 0.529 & \textbf{0.824}\\

SCC A2 & 0.471 & \textbf{0.763}\\

\hline
Average SCC & 0.371 & \textbf{0.794}\\

\hline
PCC A1 & 0.615 & \textbf{0.865}\\

PCC A2 & 0.462 & \textbf{0.893}\\
\hline
Average PCC & 0.538 & \textbf{0.879}\\
\bottomrule
\end{tabular}
\caption{Statistical correlations between the learned reward functions and the ground-truth rewards}
\label{tab: correlation}
\end{table}

\section{Conclusion}

Drawing on iterative reasoning models and cumulative prospect theory, we proposed a new game-theoretic framework (BRSMG) that captures two aspects of realistic human behaviors: bounded intelligence and risk-sensitivity. We provided general solutions to both the forward policy design problem and the inverse reward learning problem with theoretical analysis and simulation verification. Our future work will focus on using the proposed framework for practical applications such as learning human driver reward functions from naturalistic driving data.

\section*{Acknowledgements}
We thank Ruichao Jiang for helpful discussion and feedback.

\bibliography{Ref}
\clearpage


\end{document}